%% file: ms.tex
\pdfoutput=1
\documentclass[11pt]{article}
\usepackage[utf8]{inputenc}
\usepackage{hyperref}
\hypersetup{
	bookmarks=true,
	colorlinks,
	linkcolor=blue,
	citecolor=blue,
	urlcolor=blue,
}
\input{./defs}

\usepackage[letterpaper, margin=1in]{geometry}

\usepackage[backend=biber, style=ieee, doi=false, isbn=false]{biblatex}
\author{Andrew D.\ McRae \qquad Mark A.\ Davenport\\Georgia Institute of Technology}
\title{Low-rank matrix completion and denoising under Poisson noise}
\date{}

\begin{document}
\maketitle
\begin{abstract}
This paper considers the problem of estimating a low-rank matrix from the observation of all or a subset of its entries in the presence of Poisson noise. When we observe all entries, this is a problem of \emph{matrix denoising}; when we observe only a subset of the entries, this is a problem of \emph{matrix completion}. In both cases, we exploit an assumption that the underlying matrix is \emph{low-rank}. Specifically, we analyze several estimators, including a constrained nuclear-norm minimization program, nuclear-norm regularized least squares, and a nonconvex constrained low-rank optimization problem.
We show that for all three estimators, with high probability, we have an upper error bound (in the Frobenius norm error metric) that depends on the matrix rank, the fraction of the elements observed, and maximal row and column sums of the true matrix. We furthermore show that the above results are minimax optimal (within a universal constant) in classes of matrices with low rank and bounded row and column sums. We also extend these results to handle the case of matrix multinomial denoising and completion.
\end{abstract}

\section{Introduction}
\input{introduction}
\section{Upper bounds}
\input{upper}
\section{Minimax lower bounds}
\input{lower}
\section{Conclusion and future work}
\input{conclusion}

\section*{Acknowledgments}
This work was supported, in part, by NSF grant CCF-1350616 and a gift from the Alfred P.\ Sloan Foundation.

\printbibliography
\end{document}

%% file: defs.tex
\usepackage{mathtools}
\usepackage{amssymb}
\usepackage{mathrsfs}
\usepackage{cleveref}
\usepackage{amsthm}

\DeclareMathOperator{\E}{\mathbf{E}}
\renewcommand{\P}{\operatorname{\mathbf{P}}}
\newcommand{\rank}{\operatorname{rank}}

\newcommand{\defeq}{\vcentcolon=}

\newcommand{\argmin}{\operatornamewithlimits{argmin}}
\newcommand{\<}{\langle}
\renewcommand{\>}{\rangle}
\newcommand{\norm}[1]{\lVert #1 \rVert}
\newcommand{\diag}{\operatorname{diag}}
\newcommand{\abs}[1]{\lvert #1 \rvert}
\newcommand{\bignorm}[1]{\left\lVert #1 \right\rVert}

\newcommand{\dkl}[2]{\operatorname{D_{KL}}(#1 \: \Vert \: #2)}
\newcommand{\projP}{\operatorname{\mathcal{P}}}
\newcommand{\opA}{\operatorname{\mathcal{A}}}

\newcommand{\R}{\mathbf{R}}

\newcommand{\var}{\operatorname{var}}
\newcommand{\card}{\operatorname{card}}

\newcommand{\indicator}[1]{\mathbf{1}_{\{ #1 \}}}

\newcommand{\poissondist}{\operatorname{Poisson}}
\newcommand{\multinomialdist}{\operatorname{Multinomial}}

\newcommand{\Mhat}{\widehat{M}}

\newtheorem{lemma}{Lemma}
\newtheorem{theorem}{Theorem}
\newtheorem{corollary}{Corollary}

\theoremstyle{definition}

\Crefname{assumption}{Assumption}{Assumptions}

%% file: introduction.tex
\subsection{Low-rank models for count data}

In this paper, we consider the problem of estimating a non-negative matrix $M \in \R^{m \times n}$ given independent observations distributed according to $\poissondist(M_{ij})$
for $(i,j) \in \Omega$,
where $\Omega$ is a (not necessarily strict) subset of $\{1, \dots, m\} \times \{1, \dots, n\}$. If we do not make an observation for every entry of the matrix, the recovery problem is, in general, ill-posed in the absence of any additional assumptions on the underlying matrix.
A common assumption for this type of problem is that the unknown matrix $M$ is \textit{low-rank};
i.e., the dimension of the spans of the columns and rows of $M$ is much smaller than the actual numbers of columns and rows.
This assumption greatly reduces the number of degrees of freedom in the model,
making the recovery problem far more tractable.
Note that even if we do observe every entry, we can still exploit the structure of the model to reduce the error due to noise.

While the problems of matrix completion and denoising have received a significant amount of attention in the settings of Gaussian noise and of small, bounded (in $\ell_2$) perturbations (e.g., \parencite{Negahban2012,Koltchinskii2011,Candes2010}), Poisson noise models have received comparatively less attention.
In this paper, we focus primarily on the Poisson model, but we also examine closely-related multinomial models;
this includes the case in which we have a single probability distribution over matrix coordinates (for which our result is a corollary of our Poisson results)
as well as the case in which we make independent multinomial observations of matrix rows.  Collectively, these models are often natural in applications where the observations arise via some form of counting process.
The ability to recover (or de-noise) a low-rank signal from noisy, count-based observations is useful in many situations. We briefly mention two examples.

One potential application area involves imaging systems.
This includes conventional cameras (which often suffer from noise in low light or with short exposures),
but also 3-D imaging methods such as X-ray computed tomography (CT) and positron emission tomography (PET),
which, in medical imaging, would greatly benefit from an improved noise/radiation dose tradeoff. In these scenarios, the Poisson noise model is natural because the observations consists of counts of particle (e.g., photon) arrivals at a detector. In many of these settings, such as when observing a periodic or slowly-varying sequence of images, a low-rank assumption on the underlying data is natural
(see, e.g., \parencite{Zhou2014} for an overview of low-rank modeling in image applications).

Another important application is topic modeling, which is a common form of dimensionality reduction for text documents. In this case, our observations consist of counts of word occurrences in a corpus of documents. 
If we suppose that these documents can be decomposed according to a small set of topics, and that within each topic documents will exhibit similar word occurrence counts, then
a low-rank assumption on the word-frequency matrix is natural.
For example, the popular PLSI model \parencite{Hofmann1999} uses a multinomial probability model parameterized by a low-rank matrix.

Low-rank models are also popular in nonnegative matrix factorization, which is commonly applied in a range of contexts where count data is common.
A wide variety of such models have been developed,
along with inference algorithms such as expectation maximization, variational Bayes, and Markov chain Monte Carlo
\parencite{Lee1999,Canny2004,Cemgil2009,Mnih2007}.
These models and algorithms have been applied to many tasks, especially recommendation systems \parencite{Ma2011,Gopalan2015}.
However, the algorithms used are nonconvex, and there is little in the way of theoretical guarantees for their performance in the Poisson or multinomial setting.

\subsection{Summary of main results for Poisson noise}
In our analysis, we assume a Bernoulli sampling of the matrix entries:
i.e., the events $\{(i,j) \in \Omega\}$ are independent with probability $p \in (0, 1]$,
and the observed Poisson random variables are independent conditioned on $\Omega$.
Note that taking $p=1$ handles the case in which we observe every entry of the matrix.
For a matrix $A$, $\norm{A}_*$, $\norm{A}$, and $\norm{A}_F$ denote the nuclear norm, operator norm, and Frobenius norm of $A$, respectively.

Let $\opA_\Omega \colon \R^{m \times n} \to \R^\Omega$ denote the
entry-wise sampling operator given by
$(\opA_\Omega(Z))_{(i,j)} = Z_{ij}$
for $(i,j) \in \Omega$.
Note that its adjoint $\opA^*\colon \R^\Omega \to \R^{m \times n}$ maps the vector $(x_{(i,j)})_{(i,j) \in \Omega} \in \R^\Omega$ to the $m \times n$ matrix whose $(i,j)$th entry is $x_{(i,j)}$ if $(i,j) \in \Omega$ and zero otherwise.

Given observations $X \sim \poissondist(\opA_\Omega(M))$,
we consider several different estimators with similar theoretical properties.
The first can be interpreted as a matrix version of the \emph{Dantzig selector}~\parencite{Candes2007}:
\begin{equation}
	\label{eq:program_dantzig}
	\Mhat^{(1)}
	= \argmin_{M' \in [0, \infty)^{m \times n}} \norm{M'}_*
	\ \text{s.t.} \ \norm{\opA_\Omega^*(X) - p M'} \leq \delta
	,
\end{equation}
where $\delta > 0$ is a parameter which we will see how to set later.
The second is a nuclear-norm-regularized least-squares type estimator:
\begin{equation}
	\label{eq:program_reg_ls}
	\Mhat^{(2)} = \argmin_{M' \in [0, \infty)^{m \times n}} \norm{\opA_\Omega^*(X) - p M'}_F^2 + \lambda \norm{M'}_*,
\end{equation}
where $\lambda > 0$ is another parameter which we will set.
The third is least-squares under an exact low-rank constraint:
\begin{equation}
	\label{eq:program_exact_lr}
	\Mhat^{(3)} = \argmin_{M' \in [0, \infty)^{m \times n}} \norm{\opA_\Omega^*(X) - p M'}_F \ \text{s.t.} \ \rank(M') \leq r,
\end{equation}
where $r$ is an upper bound on the rank of the true rate matrix $M$. This problem is not convex (and is, in general, hard to solve directly while respecting nonnegativity constraints),
but we will see later how we can address this issue without affecting its theoretical properties.

\Cref{thm:upper}, which is the main result of \Cref{sec:upper}, states that, if $M$ has rank $r$,
and hyperparameters are properly chosen,
each of the estimators $\{\Mhat^{(i)}\}_{i=1}^3$ satisfies, with high probability,
\begin{equation}
	\label{eq:simple_ub}
	\norm{M - \Mhat^{(i)}}_F
	\lesssim \sqrt{\frac{r}{p}} \widetilde{\sigma}(M) + \text{logarithmic terms}
	,
\end{equation}
where
\[
	\widetilde{\sigma}(M)
	= \max_i \sqrt{ \sum_j M_{ij} + (1-p) M_{ij}^2}
	+ \max_j \sqrt{ \sum_i M_{ij} + (1-p) M_{ij}^2}
	.
\]
In many situations (see \Cref{sec:matching}),
the logarithmic terms are negligible,
so we can approximate this result by the bound
\begin{equation}
	\label{eq:rate-basic}
	\norm{M - \Mhat^{(i)}}_F \lesssim \sqrt{\frac{r}{p}} \widetilde{\sigma}(M)
	.
\end{equation}

\Cref{sec:lower} uses two standard methods to find lower bounds on the minimax risk of any estimator
in classes of matrices with bounded row and column sums.
These results (\Cref{thm:lower-variance,thm:lower-squared}) can be summarized as follows:
over all nonnegative matrices $M$ such that $\rank(M) \leq r$, and $\widetilde{\sigma}(M) \leq \sigma$,
we have
\[
	\inf_{\Mhat} \sup_M \E_M \norm{M - \Mhat}_F \gtrsim \sqrt{\frac{r}{p}} \sigma . 
\]
Thus, Theorem~\ref{thm:upper} is optimal (up to a multiplicative constant and an additive logarithmic factor) for this class of matrices.

To gain a more intuitive understanding of our result,
it is helpful to examine the formula for $\widetilde{\sigma}$.
For simplicity, assume, without loss of generality, that the row sums dominate the column sums, so that
\[
	\widetilde{\sigma} \approx \max_i \sqrt{\sum_j M_{ij} + (1-p) M_{ij}^2}
	.
\]
The two terms inside the sum have different roles.
The first term ($M_{ij}$) corresponds to the variance of the Poisson random variables.
Indeed, if we take $p=1$, this is the only term, so
our result has the form
\[
	\norm{M - \Mhat^{(i)}}_F \lesssim \sqrt{r} \left( \max_i \sqrt{ \sum_j M_{ij}} \right)
	.
\]
If we do not impose any structure on the model,
the maximum likelihood (and least-squares) estimate is $\Mhat^{\text{MLE}} = X$,
which has risk
\[
	\E \norm{M - \Mhat^{\text{MLE}}}_F^2 = \sum_{i,j} \var(X_{ij}) = \sum_{i,j} M_{ij}
	.
\]
If every row of $M$ has approximately the same sum,
estimators defined above improve on $\Mhat^{\text{MLE}}$ (in squared Frobenius error) by a factor of approximately $r/n$.
If the sums of the rows of $M$ differ significantly,
the improvement is smaller.  However, this should not be too surprising --- if the variance in the problem is already concentrated into a smaller sub-matrix,
we are effectively solving a smaller problem, and hence the low-rank assumption is less restrictive and, therefore, less beneficial.

The second term (of the form $(1-p)M_{ij}^2$) in the formula for $\widetilde{\sigma}$
corresponds to the inherent difficulty in estimating the values of a matrix due to the fact that we do not observe every entry.
This term in the lower bound applies regardless of the noise model,
even when there is no noise.
This might seem to contradict existing exact noiseless matrix completion results,
but we note here that such results make stronger assumptions (incoherence of the row and column spaces) beyond what we are assuming here.
In fact, the matrices used in the proof of \Cref{thm:lower-squared} are highly coherent.

Although this second error term is necessary for general matrices,
an interesting open problem is whether it could be entirely removed
(leaving only the variance term) when we assume additional structure (such as incoherence) on the true rate matrix.
Such a result would be a bridge between existing noisy and noiseless matrix completion literature;
the existence of exact completion for the noiseless case implies that current results for the noisy case (including this paper)
become highly suboptimal when the signal-to-noise ratio goes to infinity.
An exception is \parencite{Candes2010}, but we note that this approach is not without its own drawbacks as this approach leads to error rates which are suboptimal with respect to the rank $r$.
More recent work in this direction is \parencite{Ma,Chen};
however, these results depend strongly on the condition number of the true matrix,
and their dependence on matrix rank seems suboptimal.
Thus there is still much work to do in analyzing noisy completion of incoherent matrices.

\subsection{Summary of main results for multinomial denoising}
We can also derive an interesting result on multinomial matrix denoising as a corollary of our result on Poisson denoising.
If $P$ is a non-negative $m \times n$ matrix such that $\sum_{i,j} P_{ij} = 1$,
and we independently sample $N$ objects according to the probabilities contained in $P$,
the number of times each entry of $P$ is sampled (which we can denote by an $m \times n$ count matrix $X$)
has a (matrix) multinomial distribution.
Our results on Poisson denoising apply by considering a multinomial distribution to be a vector of independent Poisson variables conditioned on its sum.
\Cref{cor:multinomial} in \Cref{sec:multinomial}
shows that if $P$ has rank $r$, and none of the row and column sums of $P$ are too large,
there is an estimator $\widehat{P}$ of $P$
(which could be defined similarly to any of the estimators above)
such that
\[
	\norm{\widehat{P} - P}_F \lesssim \sqrt{\frac{r}{N(m \wedge n)}}
\]
One can easily check that, for the maximum likelihood estimator $\widehat{P}^{\text{MLE}} = X/N$,
$\E \norm{\widehat{P}^{\text{MLE}} - P}_F^2 \approx N^{-1}$,
so our result (approximately) reduces the squared error by a factor of $r / (m \wedge n)$,
which is the effective rank deficiency.

For a more complete exploration of multinomial matrix denoising,
we also consider a model in which our observations are independent multinomial samples from rows of a low-rank matrix.
Concretely, our observations are now a matrix $X$ whose \textit{rows},
which we denote $\{X_i\}_{i=1}^m$,
are independent and distributed according to $X_i \sim \multinomialdist(p_i, N_i)$,
where $\{p_i\}_{i=1}^m$ are the rows of a rank-$r$ $m \times n$ matrix $P$.
\Cref{thm:multinomial_rows} states that,
under mild conditions on the sums of \textit{columns} of $P$,
there is an estimator $\widehat{P}$
(defined similarly to those above) such that, with high probability,
\[
	\norm{D^{1/2} (\widehat{P} - P)}_F \lesssim \sqrt{r \log (m + n)},
\]
where $D = \diag(N_1, \dots, N_m)$.
It is easily checked that the maximum likelihood estimator $\widehat{P}_{\text{MLE}} = D^{-1} X$ has expected error $\E \norm{D^{1/2}(\widehat{P}^{\text{MLE}} - P)}_F^2 \approx m$,
so we again get a reduction in (squared) error that is (approximately, modulo a logarithmic factor) proportional to the reduction in degrees of freedom.

We do not analyze the multinomial estimation problems from a minimax risk standpoint,
but, due to the similarities between the Poisson and multinomial distributions,
we suspect that one could find similar matching lower bounds in a similar manner to the Poisson case. %\textbf{[Should we try to do this?]}.

\subsection{Computation and implications for general noisy matrix completion}
\label{sec:implications}
Note that all three of our estimators would be very easy to compute if we discarded the nonnegativity constraint:
we could take a singular value decomposition of the matrix $\opA_\Omega^*(X)$
and then either do singular value soft thresholding (for $\Mhat^{(1)}$ and $\Mhat^{(2)}$) or truncate it to the $r$ largest singular values (for $\Mhat^{(3)}$).

We claim that ignoring the nonnegativity constraint does not change our analysis or the resulting error bounds at all;
this constraint does not appear anywhere in the proof of \Cref{thm:upper}.
Therefore, if computational ability is a limiting factor, we could simply take the more efficient approach of solving without any nonnegativity constraints, and the error bounds we have presented will still apply.
Projecting the result onto any convex constraint set that contains $M$ can then only improve the performance.

We also note here that although we have chosen to focus on Poisson noise,
our approach is fairly general and could apply to other types of noise.
Indeed, if $M$ were an arbitrary (not necessarily nonnegative) matrix,
and we make observations of the form $M_{ij} + \xi_{ij}$ for $(i,j) \in \Omega$,
and the $\xi_{ij}$'s are zero-mean noise variables with reasonably light tails,
we could adapt our arguments to show that, for each of our three estimators,
\[
	\norm{M - \Mhat}_F \lesssim \sqrt{\frac{r}{p}} \left( \max_i \sqrt{ \sum_j \var(\xi_{ij}) + (1-p) M_{ij}^2} + \max_j \sqrt{ \sum_i \var(\xi_{ij}) + (1-p) M_{ij}^2} \right).
\]
The lower bound \Cref{thm:lower-squared} is completely independent of the noise distribution;
a version of \Cref{thm:lower-variance} could be proved for many common distributions.

The above two points have some interesting implications for general matrix completion with noise.
Many of the existing algorithms, such as low-rank factorization \parencite{Keshavan2010}, iterative imputation \parencite{Mazumder2010}, or the many other algorithms, including \eqref{eq:program_dantzig}, that can be expressed as semidefinite programs, are fairly complex.
Our results suggest that, not only do simple SVD-based algorithms have theoretical properties that are just as good as current state-of-the-art guarantees for more complex algorithms,
but that, in a minimax error sense, \textit{it is impossible to do any better}.

This realization does not, however, imply that there is no value to more sophisticated algorithms.
As mentioned earlier, how well we can exploit incoherence in \textit{noisy} matrix completion remains an important open question,
and the matrices used in the proof of the minimax lower bound \Cref{thm:lower-squared} are highly coherent.
Therefore, it is likely that more sophisticated algorithms are still beneficial when trying to recover non-pathological (i.e., incoherent) matrices.

\subsection{Comparison to prior work}
\label{sec:lit}
There are several categories of existing literature to which we can compare our results.
Some papers explicitly consider Poisson noise, using a maximum-likelihood framework.
\textcite{Cao2016} consider nuclear-norm penalized maximum likelihood for matrices contained in a nuclear norm ball
(rather than exactly low-rank matrices).
This approach uses an empirical process argument
to bound the Kullback-Leibler divergence between the true and predicted distributions.
This argument requires a Lipschitz condition on the log-likelihood function,
which, for the Poisson distribution, requires imposing a lower bound on the rates.
\textcite{Soni2016} and \textcite{Soni2014} consider a penalized maximum likelihood estimator from a carefully-chosen finite set of candidates
(which is exponentially large in the size of the problem and hence computationally intractable).
The matrices considered have a non-negative low-rank factorization
(with a particular emphasis on the case when one factor is sparse).
They use an information-theoretic argument to bound the expected error in terms of Bhattacharyya distance.
The result of \parencite{Soni2016}, which applies to matrix completion, requires imposing a lower bound on the rates,
while that of \parencite{Soni2014}, which considers only denoising, does not.
All three papers find an upper bound on Frobenius error in terms of the statistical error metrics that they originally bound.

Other, more general approaches, are designed specifically with Frobenius-norm error in mind.
One class of methods uses ``restricted strong convexity'' arguments,
which were introduced by \textcite{Negahban2012}.
These methods rely on approximating the Frobenius norm in certain restricted classes of matrices (in which the error matrix must fall) using only samples of the entries.
These methods lead to simple and elegant proofs, but the concentration inequalities on which they rely require imposing uniform upper bounds in magnitude on both the true matrix entries and the estimator entries.
Other recent papers which use this type of argument include \parencite{Klopp2015,Klopp2014,Gunasekar2015}.
Another interesting paper which uses learning theory arguments to achieve a similar result is \parencite{Gaiffas2011}.

Another class of methods is direct singular value thresholding.
Papers with this approach include \parencite{Koltchinskii2011} (which inspired our approach) as well as \parencite{Klopp2011,Chatterjee2015,Gaiffas2017}.
These methods are very simple and lend themselves to simple proofs.

An interesting blend of techniques can be seen in the papers \parencite{Lafond2015,Gunasekar2014,Robin2018},
which combine some of the general approaches mentioned above with maximum likelihood estimation for exponential families of distributions.
These methods,
like those in \parencite{Cao2016} and \parencite{Soni2016}, are difficult to apply to the Poisson distribution without imposing a lower bound on rates
because, as the mean $\lambda$ of the distribution goes to $0$,
the ``natural parameter'' $\log \lambda$ goes to $-\infty$,
whereas the general methods used require parameters to be bounded.
They also require (approximate) low rank in the matrix of natural parameters.
In the Poisson case, this is equivalent to assuming a bound on the rank of the matrix $[\log M_{ij}]$ of elementwise logarithms of the means,
which is somewhat non-standard, and certainly not the same as bounding the rank of the original matrix $M$.

There is some previous work on the denoising problem in terms of Poisson or exponential family principal component analysis (PCA).
Papers in this area include \parencite{Collins2001}, which recommends maximum likelihood approaches;
\parencite{Chiquet2018}, which uses variational Bayesian inference;
and \parencite{Liu2018}, which uses a singular value shrinkage algorithm on the means (much like we do).
The recent preprint \cite{Kenney2019} examines a variety of models with random scaling factors and nonlinearities.
These papers do not contain theoretical results applicable to our problem, however.
\parencite{Liu2018} is related to \parencite{Dobriban2020}, which contains an asymptotic analysis of a similar singular value shrinkage algorithm for more general problems (including matrix completion).
Another work in this area is \parencite{Chen2015},
which contains consistency results for low-dimensional subspace recovery.

Most of the papers mentioned above do not find error bounds
which explicitly depend on the ``true'' rate matrix;
rather, they find uniform upper bounds for classes of structured matrices with uniform upper (and, sometimes, lower) bounds on the entries.
To compare our results directly to this literature,
we consider what we obtain when we only impose a uniform upper and lower bounds (by, say $\lambda_{\max}$ and $\lambda_{\min}$) on the matrix entries.
The approximate bound of \eqref{eq:rate-basic} reduces to
\[
	\norm{\Mhat - M}_F^2 \lesssim (\lambda_{\max} + (1-p) \lambda_{\max}^2) \frac{r m }{p},
\]
where we have assumed, without loss of generality, that $m \geq n$.
Previous results show similar error rates in terms of matrix dimensions for exactly low-rank matrices.
For example, \parencite{Soni2016} establishes a bound of
\[
	\E \norm{\Mhat - M}_F^2 \lesssim \frac{\lambda_{\max}^3}{\lambda_{\min}} \frac{r m}{p} \log m,
\]
which provides a similar dependence on $r$, $m$, and $p$,
but with an additional logarithmic term and a worse dependence on the minimum and maximum matrix values.
In a slightly different setting, \parencite{Cao2016} shows that for matrices in the nuclear norm ball of radius $\lambda_{\max} \sqrt{rmn}$
(which is a convex relaxation of the exact low-rank constraint),
we instead obtain (ignoring logarithmic terms and a complicated but severe dependence on $\lambda_{\max}$ and $\lambda_{\min}$) an error bound of
\[
	\norm{\Mhat - M}_F^2 \lesssim
	\frac{\sqrt{rn} m}{\sqrt{p}},
\]
where $p$ is now the number of samples for entry in a uniform-at-random sampling model.
The different dependence on $r$ and $p$ is interesting,
but, if one compares it to results in linear regression over $\ell_1$ balls (see, e.g., \parencite{Raskutti2011}),
the rate given is perhaps not surprising.
To compare to some of the more general methods mentioned above, we note that,
if we consider the generalization of our method mentioned in \Cref{sec:implications},
and we assume a uniform upper bound on the magnitudes of matrix elements and the noise variances,
our results are comparable to \parencite{Klopp2015} (albeit under less-strict assumptions).

As noted earlier, our paper uses fairly general matrix completion methods.
For the Frobenius norm error metric,
this gives us an advantage over more distribution-specific approaches such as \parencite{Cao2016,Soni2016,Lafond2015,Gunasekar2014},
in part because we do not have to approximate the Frobenius-norm error by a statistical divergence measure or by a norm in a transformed parameter space.
Our results also do not suffer from the fact that a Poisson distribution's likelihood function is ill-conditioned for very small rates.
In addition, our results avoid a multiplicative logarithmic factor that appears in much of the previous literature
(replacing it with an additive factor that is often negligible);
this achievement (which also appears in \parencite{Klopp2015}) is almost entirely due to the use of recent results in bounding the operator norm of a random matrix (such as \parencite{Bandeira2016}).

Finally, much of the previous literature in the Poisson case
(from those mentioned above, \parencite{Lafond2015, Cao2016, Soni2016})
finds lower bounds on minimax risk in certain classes of matrices.
Although these lower bounds have the same large-scale error rate
(in terms of the rank and dimensions of the matrix and the number of samples) as the corresponding upper bounds,
they differ from the upper bounds by factors that are logarithmic in the problem size and that depend on the ratio of largest to smallest allowable rates.
To our knowledge, the results in this paper are the first for noisy low-rank matrix completion
in which the minimax rate for large classes of matrices is found to within a universal constant.

We are aware of much less theoretical work for low-rank denoising.
One recent work that is worth noting is \parencite{Huang2018}.
This paper shows that, in the case of a matrix multinomial distribution, one can achieve a tight error bound in $\ell_1$ distance (sum of absolute values of entries)
for certain factorizable probability matrices
using $O(m r^4)$ samples.
It is difficult to compare this directly to our result,
since it is in the much stronger $\ell_1$ norm,
but we note that, in a similar fashion as many of the results on Poisson observations, this paper also relies on lower bounding certain sums of entries in the factor matrices.
Other theoretical work on topic modeling includes \cite{Arora2013,Anandkumar2014,Bansal2014,Ke2017,Bing2018}.

We add a final caveat to our results by noting that $\norm{\Mhat - M}_F$ might not always be the most appropriate error metric;
for example, there is a much larger difference qualitatively (and quantitatively, if we use an appropriate statistical divergence)
between Poisson distributions of means 0 and 10 than between Poisson distributions of means 100 and 110.
We see a similar disconnect between squared error and other probabilistic metrics in the case of the multinomial distribution.
Further investigation of distribution-specific methods (such as maximum likelihood) that yield bounds in more statistically-motivated metrics is thus certainly warranted.

\subsection{Paper outline}
The remainder of this paper is organized as follows. 
\Cref{sec:both_upper} contains the the formal statements and proofs of the upper bounds on error in this paper.
\Cref{sec:lower} contains the statements and proofs of two separate minimax lower bounds which,
when combined, yield the matching lower bound to \eqref{eq:rate-basic}.
\Cref{sec:matching} also discusses briefly some situations where the approximation of \eqref{eq:rate-basic} is accurate,
and thus the upper and lower bounds match within a universal constant.

%% file: upper.tex
\label{sec:both_upper}
\subsection{Poisson noise}
\label{sec:upper}
This section is dedicated to proving the main result of this paper, which is the following theorem.
\begin{theorem}
	\label{thm:upper}
	Let $M$ be a non-negative $m \times n$ matrix with rank $r$.
	Let $\lambda_{\text{max}} = \max_{ij} M_{ij}$,
	and let
	\begin{align*}
		\widetilde{\sigma}(M)
		= \max_i \sqrt{ \sum_j M_{ij} + (1-p) M_{ij}^2}
		+ \max_j \sqrt{ \sum_i M_{ij} + (1-p) M_{ij}^2}
		.
	\end{align*}
	Suppose $\Omega \subset \{1, \dots, m\} \times \{1, \dots, n\}$
	is chosen according to a Bernoulli sampling model
	with sampling probability $p$,
	and suppose, conditionally on $\Omega$, $X \sim \poissondist(\opA_\Omega(M))$.
	Set $\epsilon \in (0, 1/2)$,
	and let
	\begin{equation}
	\label{eq:opnorm_ub}
		A(M, p, \epsilon) = 2\sqrt{p} \widetilde{\sigma}(M) + \frac{8 \epsilon}{\sqrt{mn}}
			+ C \max\left\{ \lambda_{\mathrm{max}}, 4 \log \frac{2mn}{\epsilon} \right\} \sqrt{\log \frac{m \vee n}{\epsilon}},
	\end{equation}
	where $C$ is a universal constant.
	
	Then, with probability at least $1 - 2\epsilon$, if $\delta \geq A(M, p, \epsilon)$ and $\lambda \geq 2p A(M, p, \epsilon)$,
	we also have
	\[
	 	\norm{\Mhat^{(1)} - M}_F \leq \frac{4\sqrt{2r} \delta}{p}
	\]
	and
	\[
	 	\norm{\Mhat^{(2)} - M}_F \leq \frac{2 \sqrt{2r} \lambda}{p^2}.
	\]
	Moreover, we also have that with probability at least $1- 2\epsilon$,
	\[
		\norm{\Mhat^{(3)} - M}_F \leq \frac{2 \sqrt{2r}}{p} A(M, p, \epsilon).
	\]
\end{theorem}
The result follows from a series of lemmas.
The first steps in upper bounding the error are the following (deterministic) results.
\begin{lemma}
	\label{lem:basic_1}
	Suppose $M$ is a rank-$r$ matrix such that $\norm{\opA_\Omega^* (X) - p M} \leq \delta = \frac{\lambda}{2p}$.
	Then, for $i \in \{1, 2\}$,
	\begin{equation}
		\norm{\Mhat^{(i)} - M}_F \leq \frac{4 \sqrt{2r} \delta}{p} = \frac{2 \sqrt{2r} \lambda}{p^2}.
	\end{equation}
	
\end{lemma}

\begin{lemma}
	\label{lem:basic_2}
	Suppose $M$ is a rank-$r$ matrix.
	Then
	\[
		\norm{\Mhat^{(3)} - M}_F \leq \frac{2 \sqrt{2r}}{p} \norm{\opA_\Omega^*(X) - p M}.
	\]
\end{lemma}

\begin{proof}[Proof of \Cref{lem:basic_1}]
	Let $M = U \Sigma V^*$ be the singular value decomposition of $M$,
	where $U \in \R^{m \times r}$ and $V \in \R^{n \times r}$ are such that $U^* U = V^* V = I_r$,
	and $\Sigma$ is an $r \times r$ diagonal matrix with positive entries on the diagonal.
	Let $T$ be the subspace of $\R^{m \times n}$ spanned by matrices of the form
	$U A$ and $B V^T$ for arbitrary matrices $A \in \R^{r \times n}$ and $B \in \R^{m \times r}$.
	We denote by $\projP_T$ and $\projP_{T^\perp}$, respectively, the orthogonal projections onto $T$ and its orthogonal complement $T^\perp$.
	
	Denote $H^{(1)} = \Mhat^{(1)} - M$.
	Because $M$ is feasible,
	and the nuclear norm is a convex function,
	we have 
	\[
		0 \geq \norm{\Mhat^\delta}_* - \norm{M}_*
		\geq \< H^{(1)}, Z\>
		,
	\]
	where $Z \in \partial \norm{M}_*$ is any subgradient of the nuclear norm function at the point $M$.
	Such a subgradient must have the form
	\[
		Z = U V^* + \projP_{T^\perp} (W)
		,
	\]
	where $W$ is an arbitrary matrix with $\norm{W} \leq 1$.
	By the duality of the nuclear norm and operator norm,
	we can choose $W$ so that $\<W, \projP_{T^\perp} (H^{(1)})\> = \norm{\projP_{T^\perp} (H^{(1)})}_*$.
	We then have
	\begin{align*}
		0
		&\geq \< H^{(1)}, UV^* + \projP_{T^\perp} (W) \> \\
		&= \< \projP_T (H^{(1)}), U V^*\> + \norm{\projP_{T^\perp} (H^{(1)})}_* \\
		&\geq -\norm{\projP_T (H^{(1)})}_* + \norm{\projP_{T^\perp} (H^{(1)})}_*
		,
	\end{align*}
	where the last inequality follows from the fact that $\norm{U V^*} = 1$.
	We therefore have
	\[
		\norm{\projP_{T^\perp} (H^{(1)})}_* \leq \norm{\projP_T (H^{(1)})}_*
		.
	\]
	Hence
	\begin{align*}
		\norm{H^{(1)}}_*
		&\leq 2 \norm{\projP_T (H^{(1)})}_* \\
		&\leq 2 \sqrt{2r} \norm{\projP_T (H^{(1)})}_F \\
		&\leq 2 \sqrt{2r} \norm{(H^{(1)})}_F
		,
	\end{align*}
	where the second inequality follows from the fact that any element of $T$ has rank at most $2r$.
	
	By the triangle inequality and the homogeneity of the operator norm,
	we also have
	\[
		\norm{H^{(1)}} \leq \frac{2 \delta}{p}
		.
	\]
	Thus
	\begin{align*}
		\norm{H^{(1)}}_F^2
		&= \< H^{(1)}, H^{(1)}\> \\
		&\leq \norm{H^{(1)}} \norm{H^{(1)}}_* \\
		& \leq \frac{4 \sqrt{2r} \delta}{p} \norm{H^{(1)}}_F
		,
	\end{align*}
	and the first part of the result immediately follows.
	
	The proof of the second part is similar;
	letting $H^{(2)} = \Mhat^{(2)} - M$, we now have, by the optimality of $\Mhat$,
	\begin{align*}
		0 &\geq \norm{\opA_\Omega^*(X) - p\Mhat^{(2)}}_F^2 + \lambda \norm{\Mhat^{(2)}}_* - \left( \norm{\opA_\Omega^*(X) - pM}_F^2 + \lambda\norm{M}_* \right) \\
		&= p^2 \norm{\Mhat^{(2)}}_F^2 - p^2 \norm{M}_F^2 + p \< \opA_\Omega^*(X), M - \Mhat^{(2)} \> + \lambda \left( \norm{\Mhat^{(2)}}_* - \norm{M}_* \right) \\
		&= p^2 \norm{H^{(2)}}_F^2 + 2p \< \opA_\Omega^*(X) - pM, H^{(2)} \> + \lambda \left( \norm{\Mhat^{(2)}}_* - \norm{M}_* \right).
	\end{align*}
	Noting, as before, that $\norm{\Mhat^{(2)}}_* - \norm{M}_* \geq \norm{\projP_{T^\perp} (H^{(2)})}_* - \norm{\projP_T (H^{(2)})}_*$,
	and that
	\[
		\abs{ \< \opA_\Omega^*(X) - pM, H^{(2)} \> } \leq \norm{\opA_\Omega^*(X) - pM} \norm{H^{(2)}}_*
		\leq \frac{\lambda}{2p} \norm{H^{(2)}}_*,
	\]
	we have
	\begin{align*}
		\norm{H^{(2)}}_F^2
		&\leq \frac{\lambda}{p^2}( \norm{H^{(2)}}_* + \norm{\projP_T (H^{(2)})}_* - \norm{\projP_{T^\perp} (H^{(2)})}_* ) \\
		&\leq \frac{2 \lambda}{p^2} \norm{\projP_T (H^{(2)})}_* \\
		&\leq \frac{2 \sqrt{2r} \lambda}{p^2} \norm{H^{(2)}}_F.
	\end{align*}
\end{proof}
%The second proof is even easier:

\begin{proof}[Proof of \Cref{lem:basic_2}]
	Because both $M$ and $\Mhat^{(3)}$ have rank at most $r$,
	$H^{(3)} = \Mhat^{(3)} - M$ has rank at most $2r$.
	The optimality of $\Mhat^{(3)}$ implies
	\begin{align*}
		0
		&\geq \norm{\opA_\Omega^*(X) - p \Mhat^{(3)}}_F^2 - \norm{\opA_\Omega^*(X) - p M}_F^2 \\
		&= 2p \< \opA_\Omega^*(X) - pM, H^{(3)} \> + p^2 \norm{H^{(3)}}_F^2, 
	\end{align*}
	so
	\begin{align*}
		\norm{H^{(3)}}_F^2
		&\leq \frac{2}{p} \abs{\< \opA_\Omega^*(X) - pM, H^{(3)} \>} \\
		&\leq \frac{2}{p} \norm{\opA_\Omega^*(X) - pM} \norm{H^{(3)}}_* \\
		&\leq \frac{2\sqrt{2r}}{p} \norm{\opA_\Omega^*(X) - pM} \norm{H^{(3)}}_F.
	\end{align*}
\end{proof}

The remainder of the work is to show that $\norm{\opA_\Omega^*(X) - pM} \leq A(M, p, \epsilon)$ with probability at least $1 - 2\epsilon$.
We will use the following fundamental lemma,
which was originally proved by \textcite{Bandeira2016}
and appears with a slightly improved constant in \parencite{Latala2018}.
\begin{lemma}[Theorem 4.9 and Remark 4.11 in \parencite{Latala2018}]
	\label{lem:opnorm-handel}
	Let $X$ be a random $m \times n$ matrix whose entries are independent, centered,
	and almost surely bounded in absolute value by a constant $b$.
	Let
	\[
		\sigma = \max_i \sqrt{\sum_j \E X_{ij}^2} + \max_j \sqrt{\sum_i \E X_{ij}^2}
		.
	\]
	Then
	\begin{equation*}
		\P ( \norm{X} \geq 2 \sigma + t ) \leq (m \vee n) \exp\left( - \frac{t^2}{C_0 b^2} \right)
		,
	\end{equation*}
	where $C_0$ is a universal constant.
\end{lemma}

Poisson random variables are clearly unbounded,
so \Cref{lem:opnorm-handel} does not directly apply.
The following technical lemma allows us to extend the result to the case of random variables with sub-exponential tails.
\begin{lemma}
	\label{lem:subexp}
	Let $X$ be a random $m \times n$ matrix whose entries are independent and centered,
	and suppose that for some $v, t_0 > 0$, we have, for all $t \geq t_0$,
	\[
		\P(\abs{X_{ij}} \geq t) \leq 2 e^{-t/v}
		.
	\]
	Let $\epsilon \in (0, 1/2)$,
	and let
	\[
		K = \max \left\{t_0, v \log \frac{2mn}{\epsilon} \right\}
		.
	\]
	Then
	\begin{align*}
		\P\left(\norm{X} \geq 2 \sigma + \frac{\epsilon v}{\sqrt{mn}} + t\right)
		\leq (m \vee n) \exp\left( - \frac{t^2}{C_0(2K)^2} \right)
		+ \epsilon
		,
	\end{align*}
	where $\sigma$ and $C_0$ are the same as in \Cref{lem:opnorm-handel}.
\end{lemma}
\begin{proof}
	First, note that, by a union bound,
	\begin{align*}
		\P\left( \max_{i,j} \abs{X_{ij}} > K\right)
		\leq 2 mn e^{-K/v}
		\leq \epsilon
		.
	\end{align*}
	Consider the truncation
	$X^K = [X^K_{ij}]$,
	where $X^K_{ij} = X_{ij} \indicator{\abs{X_{ij}} \leq K}$.
	Note that
	\begin{align*}
		\abs{\E X^K_{ij}}
		&\leq \E \abs{X^K_{ij} - X_{ij}} \\
		&= \E \abs{X_{ij}} \indicator{\abs{X_{ij}} > K} \\
		&= \int_K^\infty \P(\abs{X_{ij}} > t) \ dt \\
		&\leq \int_K^\infty 2 e^{-t/v} \ dt \\
		&= 2 v e^{-K/v} \\
		&\leq \frac{\epsilon v}{mn} \\
		&\leq K
		.
	\end{align*}
	Let $\widetilde{X}^K = X^K - \E X^K$ be the centered version of $X^K$.
	Clearly, $\E(\widetilde{X}^K_{ij})^2 \leq \E X_{ij}^2$,
	and $\abs{\widetilde{X}^K_{ij}} \leq 2K$.
	Then, by \Cref{lem:opnorm-handel},
	\[
		\P(\norm{\widetilde{X}^K} \geq 2 \sigma + t)
		\leq (m \vee n) e^{-t^2 / C_0 (2K)^2}
		.
	\]
	Furthermore, with probability at least $1 - \epsilon$,
	\begin{align*}
		\norm{X}
		&= \norm{X^K} \\
		&\leq \norm{\widetilde{X}^K} + \norm{\E X^K} \\
		&\leq \norm{\widetilde{X}^K} + \norm{\E X^K}_F \\
		&\leq \norm{\widetilde{X}^K} + \frac{\epsilon v}{\sqrt{mn}}
		,
	\end{align*}
	and the result follows.
\end{proof}
To apply this result, we need a subexponential tail bound for the Poisson distribution.
\begin{lemma}
	\label{lem:pois-tail}
	Let $X \sim \poissondist(\lambda)$.
	Then
	\[
		\P(X - \lambda \geq t) \leq \exp\left( - \frac{t^2}{2(\lambda + t/3)} \right)
		.
	\]
	For $t \geq \lambda$,
	\[
		\P(X - \lambda \geq t) \leq e^{-3t/8}
		.
	\]
\end{lemma}
The first inequality can be established by approximating the Poisson distribution with mean $\lambda$
as the sum of $k$ Bernoulli random variables with mean $\lambda/k$,
applying Bernstein's inequality,
and taking $k \to \infty$.
The idea for this argument was suggested by an exercise in \parencite{Pollard2016}.

Going back to our original problem,
we need to bound the operator norm of $Z = \opA_\Omega^* (X) - p M$.
Note that since we are using a Bernoulli sampling model,
the entries of $Z$ are independent.
Let $\lambda_{\text{max}} = \max_{i,j} M_{ij}$.
Note that for every $(i,j)$, $\E Z_{ij} = 0$,
\[
	Z_{ij} \geq -p M_{ij} \geq - \lambda_{\text{max}}
	,
\]
and, for $t \geq 2 \lambda_{\text{max}}$,
\[
	\P(Z_{ij} \geq t) \leq e^{-3(t-\lambda_{\text{max}})/8} \leq e^{-3t/16}
	\leq e^{-t/8}
	.
\]
Then, by \Cref{lem:subexp}, we have, for $\epsilon \in (0, 1/2)$,
\begin{align*}
	\P\left(\norm{X} \geq 2 \sigma + \frac{8 \epsilon}{\sqrt{mn}} + t\right)
	\leq (m \vee n) \exp\left( - \frac{t^2}{C_0(2K)^2} \right)
	+ \epsilon
	,
\end{align*}
where
\[
	K = \max\left\{2 \lambda_{\text{max}}, 8 \log \frac{2mn}{\epsilon} \right\}
	,
\]
and $\sigma$ is defined as before.
To calculate $\sigma$ in terms of $p$ and $M$,
we note that
\begin{align*}
	\var(Z_{ij})
	&= p \var(X_{ij}) + p(1-p) (\E X_{ij})^2 \\
	&= p M_{ij} + p(1-p) M_{ij}^2
	.
\end{align*}
Therefore, we can calculate
\[
	\sigma
	= \max_i \sqrt{\sum_j p M_{ij} + p(1-p) M_{ij}^2} + \max_j \sqrt{\sum_i p M_{ij} + p(1-p) M_{ij}^2}
	= \sqrt{p} \widetilde{\sigma}.
\]

\subsection{Corollary on multinomial estimation}
\label{sec:multinomial}
Here we prove a corollary of \Cref{thm:upper}, showing how it implies a result for estimating the low-rank-matrix parameter of a matrix multinomial distribution.
For the sake of brevity, we only consider the analogue of $\Mhat^{(1)}$.
\begin{corollary}
	\label{cor:multinomial}
	Let $P$ be a nonnegative $m \times n$ matrix with rank $r$ such that $\sum_{i,j} P_{ij} = 1$.
	Suppose, furthermore, that we have
	\[
		\max_i \sum_j P_{ij} \leq \frac{a}{m}, \max_j \sum_i P_{ij} \leq \frac{b}{n}
	\]
	for some constants $a,b \geq 1$,
	and that $\max_{i,j} P_{ij} \leq c$.
	Let $N$ be a positive integer, and suppose that $X \sim \multinomialdist(P, N)$.
	Let $\epsilon \in (0,1)$, choose $\delta > 0$ such that
	\[
		\delta \geq \frac{1}{N} \left( 2 \sqrt{N\left(\frac{a}{m} + \frac{b}{n} \right)} + \frac{4 \epsilon}{e \sqrt{mnN}} + C \max\left\{ Nc, 4 \log \frac{4e mn \sqrt{N}}{\epsilon} \right\} \sqrt{\log \frac{2 e \sqrt{N} (m \vee n)}{\epsilon} } \right)
		,
	\]
	and let
	\[
		\widehat{P}^\delta(X) = \argmin_{\substack{P' \in [0, 1]^{m \times n} \\ \sum_{i,j} P_{ij} = 1}} \norm{P'}_*
		\ \mathrm{s.t.}\ \norm{X - NP'} \leq N\delta
		.
	\]
	Then, with probability at least $1 - \epsilon$,
	\[
		\norm{\widehat{P}^\delta - P}_F
		\leq 4\sqrt{2r} \delta
		.
	\]
\end{corollary}
As in the Poisson case, there are many situations where the additive logarithmic term in the definition of $\delta$ is negligible,
and we have
\[
	\norm{\widehat{P}^\delta - P}_F \lesssim \sqrt{\frac{r}{N} \left(\frac{a}{m} + \frac{b}{n} \right)}.
\]
\begin{proof}[Proof of \Cref{cor:multinomial}]
	Suppose $Y \sim \poissondist(NP)$.
	Let $\epsilon' = \epsilon / 2 e \sqrt{N}$.
	Define the event $A = \{ \norm{\widehat{P}^\delta(Y) - P}_F \leq 4 \sqrt{2r}\delta \}$.
	\Cref{thm:upper} (with $p = 1$) implies that $\P(A) \geq 1 - 2 \epsilon' = 1 - \epsilon/e\sqrt{N}$. Note that the additional constraints in the optimization problem do not affect this fact, since $P$, by definition, meets these constraints.
	
	$X$ has the same distribution as $Y$ conditioned on $B = \{\sum_{i,j} Y_{ij} = N\}$,
	so it suffices to show that the probability of $A$ conditioned on this event is at least $1 - \epsilon$.
	Indeed, note that $\sum_{i,j} Y_{ij} \sim \poissondist(N)$, so
	\[
		\P(B) = \frac{e^{-N} N^N}{N!} \geq \frac{1}{e\sqrt{N}}
	\]
	by Stirling's approximation.
	Then, by Bayes' rule,
	\begin{align*}
		\P(\norm{\widehat{P}^\delta(X) - P}_F \geq 4 \sqrt{2r}\delta)
		&= \P(A^c \mid B ) \\
		&= \frac{\P(A^c \cap B)}{\P(B)} \\
		&\leq \frac{\P(A^c)}{\P(B)} \\
		&\leq 2\epsilon' e\sqrt{N} \\
		&= \epsilon
	\end{align*}
\end{proof}

\subsection{Multinomial denoising with independent rows}
We now consider a slightly different setting for multinomial estimation.
Here, we consider a model in which we observe a collection of independent multinomial random variables,
each of which has distribution parametrized by a row in a low-rank matrix.
\begin{theorem}
	\label{thm:multinomial_rows}
	Let $X_1, \dots, X_m$ be independent multinomial random variables, with
	\[
		X_i \sim \multinomialdist(p_i, N_i),
	\]
	where, for each $i$, $N_i \geq 1$ is an integer,
	and $p_i = (p_{i1}, \dots, p_{in})$ is a vector of probabilities.
	Denote the matrices
	\[
		X = \begin{bmatrix*}[c]
			X_1 \\ \vdots \\ X_m
		\end{bmatrix*} \in \R^{m \times n}, \
		P = \begin{bmatrix*}[c]
			p_1 \\ \vdots \\ p_m,
		\end{bmatrix*} \in \R^{m \times n}, \
		D = \begin{bmatrix*}[c]
			N_1 & & \\
			& \ddots & \\
			& & N_m
		\end{bmatrix*} \in \R^{m \times m},
	\]
	where each $X_i$ and $p_i$ are considered row vectors.
	Define the estimator
	\[
		\widehat{P}^\delta = \argmin_{\substack{P' \in [0, 1]^{m \times n}}}
		\norm{D^{1/2}P'}_* \ \mathrm{s.t.}\ \norm{D^{-1/2}(X - DP')} \leq \delta, \ \sum_j P_{ij} = 1 \ \forall i \in \{1, \dots, m\}.
	\]
	Let $D_{\mathrm{min}} = \min_i D_i$, and choose $\delta$ such that
	\[
		\delta \geq \max\left\{ 2 \sqrt{ \max \left\{ 1, \max_j \sum_i p_{ij} \right\}\log \frac{m + n}{\epsilon}}, \frac{4}{3 \sqrt{D_{\mathrm{min}}}} \log \frac{m + n}{\epsilon} \right\}.
	\]
	Then, with probability at least $1 - \epsilon$,
	\[
		\norm{D^{1/2} (\widehat{P}^\delta - P)}_F \leq 4 \sqrt{2r} \delta,
	\]
	where $r$ is the rank of $P$.
\end{theorem}
\noindent Our approach uses the following matrix Bernstein inequality.
\begin{lemma}[{\parencite[Theorem 6.1.1]{Tropp2015}}]
	\label{lem:mat_bern}
	Suppose $Z = \sum_k S_k \in \R^{m \times n}$,
	where $\{ S_k \}_k$ is a finite sequence of independent, zero-mean random matrices.
	Suppose that for some constant $L > 0$, for each $k$, $\norm{S_k} \leq L$ almost surely.
	Let
	\[
		v = \max \{ \norm{\E Z Z^T}, \norm{\E Z^T Z} \}
		= \max \left\{ \bignorm{\sum_k \E S_k S_k^T}, \bignorm{\sum_k \E S_k^T S_k} \right\}.
	\]
	Then, for $t \geq 0$,
	\[
		\P(\norm{Z} \geq t) \leq (m + n) \exp\left( \frac{-t^2}{2(v + L t / 3)} \right).
	\]
\end{lemma}

\begin{proof}[Proof of \Cref{thm:multinomial_rows}]
	On the event that $\norm{D^{-1/2}(X - DP)} \leq \delta$,
	we have $\norm{D^{1/2}(\widehat{P}^\delta - P)} \leq 2 \delta$,
	and the result follows by the same steps as in the proof of \Cref{thm:upper}.
	
	To find a bound on the operator norm of $Z = D^{-1/2}(X - DP)$, we apply \Cref{lem:mat_bern},
	noting that $D^{-1/2}(X - DP)$ is the sum of $\sum_{i=1}^m D_i$ random, zero-mean matrices of the form $\frac{1}{\sqrt{D_i}} e_i u_{ij}^T$,
	where $e_i$ is the $i^{\text{th}}$ standard basis element in $\R^m$,
	and $u_{ij}$ is a vector in $\R^n$ that is equal to $e_k$ with probability $p_{ik}$.
	
	One can verify that
	\[
		\E Z Z^T = \diag\left( \sum_{j=1}^n p_{1j}(1-p_{1j}), \dots, \sum_{j=1}^n p_{mj}(1-p_{mj}) \right)
		\preceq I_m,
	\]
	where $I_m$ denotes the $m \times m$ identity matrix,
	and that
	\[
		\E Z^T Z = \sum_{i=1}^m \diag(p_i) - p_i p_i^T \preceq \diag\left( \sum_{i=1}^m p_{i1}(1-p_{i1}), \dots, \sum_{i=1}^m p_{in}(1-p_{in}) \right),
	\]
	where $p_i = (p_{i1}, \dots, p_{in})$ is the $i^{\text{th}}$ row of $p$.
	We can then take $v \leq \max\{ 1, \max_j \sum_i p_{ij} \}$.
	Clearly, each term in the sum has operator norm bounded above by $L = 1 / \sqrt{D_{\mathrm{min}}}$.
	
	Some algebraic manipulation of the result of \Cref{lem:mat_bern} implies that,
	for $\epsilon \in (0, 1)$, we have, with probability at least $1 - \epsilon$,
	\begin{align*}
		\norm{Z}
		&\leq \max\left\{ 2 \sqrt{v \log \frac{m + n}{\epsilon}}, \frac{4 L}{3} \log \frac{m+n}{\epsilon} \right\} \\
		&\leq \max\left\{ 2 \sqrt{ \max \left\{ 1, \max_j \sum_i p_{ij} \right\}\log \frac{m + n}{\epsilon}}, \frac{4}{3 \sqrt{D_{\mathrm{min}}}} \log \frac{m + n}{\epsilon} \right\},
	\end{align*}
	which establishes the result.
	
\end{proof}

%% file: lower.tex
\label{sec:lower}
In this section, we show that the rate in \eqref{eq:rate-basic} is optimal
(within a multiplicative constant) in the sense of minimax risk.
We do this in two parts;
the two minimax lower bounds derived in the next two sections,
when combined, match the rate in \eqref{eq:rate-basic} is optimal.

\subsection{First lower bound}
In the Poisson upper error bound, $\widetilde{\sigma}$ is partially determined by the maximal row and column sums of the rate matrix $M$,
which we can think of as the maximal variance of any row or column (without sampling a subset of the entries).
Our first lower bound shows that we cannot improve on this term:
\begin{theorem}
	\label{thm:lower-variance}
	Let $r$, $k$, and $\ell$, be positive integers,
	and take $m = r k$, $n = r \ell$.
	Let $\lambda_{\max} \geq 1/8\ell p$,
	set $\sigma_1^2 = k \lambda_{\max}$,
	and let
	\[
		%\begin{aligned}
			S = \left\{ M \in [0, \lambda_{\max}]^{m \times n} :
			\rank(M) \leq r, \sqrt{\max_i \sum_j M_{ij}} + \sqrt{\max_j \sum_i M_{ij}} \leq 2\sigma_1 \right\}
			.
		%\end{aligned}
	\]
	Then, under a Bernoulli sampling model with sampling probability $p$,
	\[
		\inf_{\Mhat} \sup_{M \in S_1} \P_M \left( \norm{\Mhat - M}_F \geq \frac{\sqrt{r} \sigma_1}{8\sqrt{2p}} \right)
		\geq \frac{1}{2} - \frac{8 \log 2}{m \vee n}
		.
	\]
\end{theorem}
\begin{proof}
	Assume, without loss of generality, that $k \geq \ell$.
	We use a variant of Fano's method.
	We first find a large hypercube of matrices,
	and then we use the fact that we can find a large subset that is well-separated.
	
	For $i \in \{1, \dots, m\}$,
	let
	\[
		A_i = \{ \ell(q-1)+1, \dots \ell q \}
		,
	\]
	where $q \in \{1, \dots, r\}$ is the unique integer such that
	$i \in \{ k(q-1)+1, \dots, k q \}$.
	For $\lambda_0, \lambda_1$ to be chosen later,
	we define, for $\theta \in \{0,1\}^m$, the block-diagonal matrix $M_\theta \in S$ by
	\[
		(M_\theta)_{ij} = \begin{cases}
			\lambda_{\theta_i}, & j \in A_i,  \\
			0, & \text{otherwise}.
		\end{cases}
	\]
	The nonzero elements of the $i^{\text{th}}$ row of $M_\theta$ are all either $\lambda_0$ or $\lambda_1$, depending on the value of $\theta_i$.
	
	By a combinatorial argument (see, e.g., \parencite{Huber1997}),
	one can show that there exists $\Theta \subset \{0, 1\}^m$
	such that $\card(\Theta) \geq e^{m/8}$
	and, for all distinct $\theta,\theta' \in \Theta$,
	the Hamming distance $d_H(\theta, \theta') \geq m/4$.
	
	Take $\lambda_0 = \lambda_{\max}/2 - \delta$ and $\lambda_1 = \lambda_{\max}/2 + \delta$,
	where $\delta \leq \lambda_{\max}/2$ is a constant to be chosen later.
	Note that for all distinct $\theta,\theta' \in \Theta$, we have
	\[
		\norm{M_\theta - M_{\theta'}}_F \geq \sqrt{m\ell} \delta
		.
	\]
	We denote by $P_\theta$ the distribution of $\projP_\Omega (X)$
	when $X \sim \poissondist(M_\theta)$,
	and $\Omega$ is independently chosen from the Bernoulli sampling model with parameter $p$.
	From Fano's inequality from information theory, one can derive (see, e.g., \parencite{Huber1997})
	a lower bound on the probability of an estimator's error exceeding half the distance between points indexed by $\Theta$:
	\begin{align*}
		\inf_{\Mhat} \sup_{M \in S} \P_M\left(\norm{\Mhat - M}_F \geq \frac{\sqrt{m\ell} \delta}{2} \right)
		&\geq p_e \\
		&\defeq \inf_{\phi} \sup_{\theta \in \Theta} \P(\phi(\projP_\Omega(X)) \neq \theta) \\
		&\geq 1 - \frac{\sup_{\theta \in \Theta} \dkl{P_\theta}{Q} + \log 2}{\log \card(\Theta)}
		,
	\end{align*}
	where $\phi$ denotes a test taking values in $\Theta$,
	and $Q$ is any probability distribution on $\R^{m \times n}$.
	We take $Q$ to be the distribution generated in the same manner as each $P_\theta$,
	simply with $\lambda_1$ and $\lambda_2$ replaced by $\lambda_{\max}/2$.
	
	Note that the Kullback-Leibler divergence between two Poisson distributions with rates $\lambda$ and $\lambda'$ is
	\begin{align*}
		\dkl{\lambda}{\lambda'}
		&= \lambda' - \lambda + \lambda \log \frac{\lambda}{\lambda'} \\
		&\leq \lambda' - \lambda + \lambda\left( \frac{\lambda}{\lambda'} - 1\right) \\
		&= \frac{(\lambda - \lambda')^2}{\lambda'}
		.
	\end{align*}
	Therefore, for any $\theta \in \Theta$,
	\begin{align*}
		\dkl{P_\theta}{Q}
		&\leq r k \ell p \frac{\delta^2}{\lambda_{\max}/2} \\
		&=m \ell p \frac{2 \delta^2}{\lambda_{\max}}
		.
	\end{align*}
	Take
	\[
		\delta = \sqrt{\frac{\lambda_{\max}}{32 \ell p}} \leq \frac{\lambda_{\max}}{2}
		.
	\]
	Then
	\[
		p_e \geq \frac{1}{2}- \frac{8 \log 2}{m}
		,
	\]
	and half the separation between points indexed by $\Theta$ is at least
	\[
		\frac{\sqrt{m\ell} \delta}{2}
		= \frac{1}{8\sqrt{2}} \sqrt{\frac{m \lambda_{\max}}{p}}
		= \frac{1}{8\sqrt{2}} \sqrt{\frac{r}{p}} \sigma_1
		.
	\]
\end{proof}

\subsection{Second lower bound}
The previous theorem relies on the fact that the observations are conditionally Poisson.
The next result, which provides the second part of a matching lower bound to \eqref{eq:rate-basic},
does not depend on the conditional distribution of the observations,
and instead shows a fundamental limit in inferring missing matrix entries.
\begin{theorem}
	\label{thm:lower-squared}
	Take again $m = rk$, $n = r\ell$.
	Set $\sigma_2^2 = k \lambda_{\max}^2$.
	Let
	\[
		%\begin{aligned}
			S = \left\{ M \in [0, \lambda_{\max}]^{m \times n} :
			\rank(M) \leq r, \sqrt{\max_i \sum_j M_{ij}^2} + \sqrt{\max_j \sum_i M_{ij}^2} \leq 2\sigma_2 \right\}
			.
		%\end{aligned}
	\]
	Suppose $p \geq \frac{1} {2(k \wedge \ell)} = \frac{r}{2(m \wedge n)}$.
	Then, under a Bernoulli sampling model with probability $p$
	(with any conditional distribution on the observations),
	\begin{align*}
		\inf_{\Mhat} \sup_{M \in S_2}
		\E \norm{\Mhat - M}_F^2
		&\geq \frac{r \sigma_2^2}{8} \max \left\{
		\frac{1}{2} \left\lfloor \frac{1}{2p} \right\rfloor,
		1-p
		\right\} \\
		&\geq \frac{1}{64} \frac{1-p}{p} r \sigma_2^2
		.
	\end{align*}
\end{theorem}
\begin{proof}
	Again, we assume that $k \geq \ell$.
	We first prove the lower bound with the first item in the maximum.
	For notational simplicity, we can assume that $1/2p$ is an integer.
	Furthermore, we can assume that $\ell = 1/2p$,
	since decreasing the number of columns does not increase risk.
	
	We consider a set of matrices $\{M_\theta\}_{\theta \in \{0, 1\}^m}$
	with the same structure as in the proof of \Cref{thm:lower-variance},
	but now, we take $\lambda_0 = 0$, $\lambda_1 = \lambda_{\max}$.
	
	Assouad's lemma (see, e.g., \parencite{Yu1997} or \parencite{Huber1997}) gives a lower bound
	on the Bayes risk of an estimator $\Mhat$ for a uniform prior on $\{0,1\}^m$:
	\begin{align*}
		R(\Mhat)
		&\defeq \frac{1}{2^m} \sum_{\theta \in \{0,1\}^m} \E_{\theta} \norm{\Mhat - M_\theta}_F^2 \\
		&\geq \frac{1}{2} \sum_{i = 1}^m \frac{\ell \lambda_{\max}^2}{4} \inf_{\phi} \left( \P_{\theta}(\phi \neq \theta_i) + \P_{\theta^i}(\phi \neq 1-\theta_i) \right)
		,
	\end{align*}
	where $\phi\colon \R^{m \times n} \to \{0,1\}$ is a test,
	and, for $\theta \in \{0,1\}^m$, $\theta^i$ denotes the element of $\{0,1\}^m$
	that is equal to $\theta$ except in the $i^{\text{th}}$ position.
	
    Denote by $P_\theta^i$ the marginal distribution of the $i^{\text{th}}$ row of a matrix with distribution $P_\theta$.
	The minimal testing risk which appears in the sum is equal to the $L_1$ norm of the minimum of the densities of $P_\theta^i$ and $P_{\theta^i}^i$,
	which measures how much the distributions overlap.
	Thus we have
	\begin{align*}
		\inf_{\phi} \left( \P_{\theta}(\phi \neq \theta_i) + \P_{\theta^i}(\phi \neq 1-\theta_i) \right)
		&=\norm{P_\theta \wedge P_{\theta^i}}_1 \\
		&\geq (1 - p)^\ell \\
		&\geq 1 - \ell p \\
		&= \frac{1}{2}
		,
	\end{align*}
	where the first inequality is due to the fact that, with probability $(1-p)^\ell$, no entry from the $i^{\text{th}}$ row of $M$ is observed.
	
	Then
	\[
		R(\Mhat)
		\geq \frac{m \ell \lambda_{\max}^2}{16}
		= \frac{r \sigma_2^2}{32 p}
		.
	\]
	The result follows from the fact that minimax risk always exceeds Bayes risk.
	
	A simple modification with $\ell = 1$ yields the result for the second term in the maximum.
\end{proof}
We note here that we could also use a similar argument as in the proof of \Cref{thm:lower-variance}
to get a high-probability lower bound on error (with a somewhat worse constant).
For example, setting $Q = \frac{P_{(0, \dots, 0)} + P_{(1, \dots, 1)}}{2}$,
it is easily verified that the resulting Kullback-Leibler divergences $\dkl{P_\theta}{Q}$ can be upper bounded by a sum of coordinate-wise total variation distance.

\subsection{When do the upper and lower bounds match?}
\label{sec:matching}
Within multiplicative constants,
the lower bounds of \Cref{thm:lower-variance,thm:lower-squared} match the approximate upper bound of \eqref{eq:rate-basic}.
We must therefore consider when the approximation in \eqref{eq:rate-basic} is accurate.

Finding technical conditions that guarantee matching rates is not something we think likely to be very instructive at this point,
especially since a different proof technique could potentially change the logarithmic term in $A(M, p, \epsilon))$.
However, we think it is helpful to look at the matrices involved in the proofs of \Cref{thm:lower-variance,thm:lower-squared}.
Note that when the bounds match for those particular matrices,
the minimax error rate bounds are tight for the matrix classes considered in those proofs.

For these matrices (assuming that $m \geq n$),
\[
	\widetilde{\sigma} \approx \sqrt{\frac{m}{r}} (\sqrt{\lambda_{\max}} + \sqrt{1-p}\lambda_{\max})
	.
\]
For this term to dominate \eqref{eq:opnorm_ub},
we must have
\[
	\widetilde{\sigma} \gtrsim \frac{ \lambda_{\max} \vee c \log m}{\sqrt{p}} \sqrt{\log m}
	.
\]
For example, if $\lambda_{\max} \geq \log m$,
it would suffice to take
\[
	p \gtrsim \frac{r \log m}{m}
	,
\]
which is a standard condition in noiseless matrix completion.
If $\lambda_{\max} \leq \log m$,
it would suffice to take
\[
	p \gtrsim \frac{r \log^3 m}{m \lambda_{\max}^2}
	.
\]

%% file: conclusion.tex
In this paper, we have derived an upper bound in Frobenius norm error for an estimator for Poisson matrix completion,
and we have derived a minimax lower bound that matches this upper bound (within a universal constant) for many classes of nonnegative rate matrices.
We have also derived similar upper bounds for error in two types of multinomial matrix denoising problems.
The estimators we use are computationally tractable,
and require significantly fewer assumptions on the underlying matrix than previous results in the literature.
Significantly, we impose no lower bounds on the entries of the underlying matrix.
This is crucial in many applications (such as topic modelling) where zero or very small means can be relatively common.

Because we have found upper and lower error bounds in Frobenius norm,
the only theoretical improvement remaining for this model and error metric in general classes of matrices is to try to relax the conditions under which the bounds match
(although, as we have seen, they are not too restrictive now).
This could potentially come about by reducing the logarithmic term in \eqref{eq:opnorm_ub}
and/or by finding a logarithmic term to add to the minimax lower bounds.
One could also further examine how to obtain better error rates for more restrictive classes of matrices, such as incoherent matrices.

It would also be interesting to extend the results presented here to matrices that are not exactly low-rank,
but are instead ``approximately low-rank''; for example, we could consider matrices which are contained in Schatten balls
(which, for $q \in [0,1]$, are sets of matrices for which $\sum_i \sigma_i^q \leq R$, where $\{\sigma_i\}$ is the set of singular values).
As mentioned previously, \textcite{Cao2016} used the Schatten $1$-norm ($q=1$, or nuclear norm ball);
\textcite{Negahban2012} also examined these classes of matrices.

Another avenue of research would be to examine structured Poisson or multinomial estimation under different, more statistically motivated error metrics.
Maximum likelihood methods seem more suitable here than least-squares,
but analysis of maximum likelihood estimators has proved difficult for the reasons outlined in \Cref{sec:lit}.
It is not clear what kind of structure would be relevant in a different error metric.
Low-rank structure seems to work well with a least-squares error framework,
but there is \emph{a priori} not much reason to think that it would work similarly well for another metric;
for example, the Bhattacharyya distance between Poisson distributions,
is proportional to the (squared) $\ell_2$ distance between the \textit{square roots} of the rates,
but the element-wise square root of a low-rank matrix is not, in general, low rank. Thus, this approach may not immediately bear much fruit. However, an analysis of matrix estimation under alternative error metrics remains an important area for future research.